\documentclass{article}
\usepackage{spconf}

\usepackage[english]{babel}
\usepackage{ifpdf}
\usepackage{cite} 
\usepackage{url}
\usepackage{hyperref}
\ifpdf
	\usepackage[pdftex]{graphicx}
	\graphicspath{{./figures/}}
\else
	\usepackage[dvips]{graphicx}
	\graphicspath{./figures/}
\fi
\usepackage{color}
\usepackage{pgf, tikz, pgfplots}
\usetikzlibrary{shapes, arrows, automata}
\usepackage{caption} 
\usepackage{subcaption} 
\usepackage{amsmath}
\usepackage{amsfonts, amssymb, amsthm}
\usepackage{mathrsfs}
\usepackage{upgreek}
\usepackage{algorithm,algpseudocode}
\usepackage{enumerate}
\usepackage{multirow}
\usepackage{rotating}
\usepackage{needspace}

\input{mysymbol.sty}

\newtheorem{assumption}{\hspace{0pt}\bf AS\hspace{-0.05cm}}
\newtheorem{lemma}{\hspace{0pt}\bf Lemma}
\newtheorem{proposition}{\hspace{0pt}\bf Proposition}

\newtheorem{theorem}{\hspace{0pt}\bf Theorem}

\newtheorem{definition}{\hspace{0pt}\bf Definition}

\newcommand{\QED}{\hfill\ensuremath{\blacksquare}}

\def\Tr{\mathsf{T}}
\def\Hr{\mathsf{H}}

\title{Graphon and graph neural network stability}
\name{Luana Ruiz, Zhiyang Wang and Alejandro Ribeiro\thanks{Support by USA NSF CCF 1717120, ARL DCIST CRA W911NF-17-2-0181. 
		}
}
\address{Department of Electrical and Systems Engineering, University of Pennsylvania, Philadelphia, USA
}
\begin{document}
\ninept
\maketitle
%


\begin{abstract}
Graph neural networks (GNNs) are learning architectures that rely on knowledge of the graph structure to generate meaningful representations of large-scale network data. GNN stability is thus important as in real-world scenarios there are typically uncertainties associated with the graph. We analyze GNN stability using kernel objects called graphons. Graphons are both limits of convergent graph sequences and generating models for deterministic and stochastic graphs. Building upon the theory of graphon signal processing, we define graphon neural networks and analyze their stability to graphon perturbations. We then extend this analysis by interpreting the graphon neural network as a generating model for GNNs on deterministic and stochastic graphs instantiated from the original and perturbed graphons. We observe that GNNs are stable to graphon perturbations with a stability bound that decreases asymptotically with the size of the graph. This asymptotic behavior is further demonstrated in an experiment of movie recommendation.


\end{abstract}
\begin{keywords}
graph neural networks, stability, graphons, graph signal processing
\end{keywords}
\section{Introduction}
\label{sec:intro}

Graph neural networks (GNNs) are layered information processing architectures alternating graph convolutional filters and pointwise nonlinearities \cite{Scarselli08-TheGNNmodel,Gama19-Architectures,Kipf17-GCN,Defferrard17-ChebNets}. Graph convolutional filters, or graph convolutions for short, are 
shift-and-sum operations akin to the time and spatial convolutions of convolutional neural networks (CNNs), but parametrized so as to couple the data with the graph through learnable weights. The coupling between graph and data is achieved by a graph matrix representation called graph shift operator (GSO)---an umbrella term for matrices that, like the adjacency or the Laplacian, encode the sparsity pattern of the graph \cite{ortega2018graph}. Thus, when implementing graph convolutions and GNNs, the underlying assumption is that the graph structure is fully known.

In practice, however, this is not always the case. In networks that need to be measured, uncertainties around the graph structure usually come from measurement noise \cite{moore2004robust}. In physical networks, they can also be caused by faulty sensors or faded communication channels \cite{5727972}. These issues are even more pronounced in large and dynamic graphs \cite{adar2007managing,chen2012vulnerability}. Two important research questions that arise are then: how is the performance of GNNs affected by changes of the underlying graph? and, what is the effect of the graph size on their stability to graph perturbations?

We answer these questions by drawing on graphons.
Graphons are kernel objects that are both limits of convergent graph sequences and generative models for graphs \cite{lovasz2006limits,borgs2008convergent, gao2018graphon,avella2018centrality}. From their graph limit interpretation, we can think of them as objects identifying families of graphs that are structurally similar in the sense that they share certain densities of motifs also shared by the graphon. Graphs associated with different graphon families are structurally different and can be seen as perturbations of one another. Hence, 
 we focus on perturbations of graphons and use their generative model interpretation to instantiate a graph and its perturbation from the original and the perturbed graphons. 
 

To carry out our stability analysis, we take a step-by-step approach relying on graphon signal processing \cite{ruiz2020graphon,morency2017signal}. This theoretical framework defines graphon signals and graphon filters as abstractions of graph data and graph filters in the limit, which we use to define graphon neural networks (WNNs) and study their stability in Section \ref{sec:wnns}. From the WNN stability result (Theorem \ref{thm:wnn_stab}), we prove stability of GNNs supported on deterministic graphs instantiated from the graphon (Theorem \ref{thm:det_graph_stab}) using the GNN-WNN approximation result from \cite[Theorem 1]{ruiz2020wnn}. This result is then extended to stochastic graphs in Theorem \ref{thm:stab_stochastic}. The main takeaways from these theorems are that: (i) there exists a trade-off between the stability and discriminability of GNNs in the form of a restriction on the passing band of the graph convolutional filters, and (ii) stability increases asymptotically with the size of the graph. To conclude, we illustrate GNN stability in a movie recommendation example in Section \ref{sec:sims}.

Most of the GNN stability literature considers perturbations acting directly on the graph. Absolute and relative graph perturbations are considered in \cite{Gama19-Stability,Gama19-Scattering}, which study the stability of GNNs and graph scattering transforms respectively, while \cite{ZouLerman19-Scattering} analyzes the stability of graph scattering transforms to permutations and perturbations of the spectra.
Other related work includes studies of the transferability of GNNs on graphs instantiated from graphons \cite{ruiz2020wnn} and generic topological spaces \cite{levie2019transferability}. Most in line with this paper, \cite{keriven2020convergence} analyzes the stability of GNNs on standard random graph models, but it does not make the spectral considerations that lead to the key stability-discriminability trade-off we observe.




\section{Preliminary Definitions}
\label{sec:prelims}

In order to analyze stability to graphon perturbations, we start by reviewing GNNs and the graphon signal processing framework.


\subsection{Graph neural networks}
 
Graphs are triplets $\bbG=(\ccalV,\ccalE,\ccalW)$, where $\ccalV$, $|\ccalV|=n$, is the set of nodes, $\ccalE \subseteq \ccalV \times \ccalV$ is the set of edges, and $\ccalW: \ccalE \to \reals$ is a function assigning weights to the edges of $\bbG$. We focus on undirected graphs, so that $\ccalW(i,j)=\ccalW(j,i)$. Data supported on networks is modeled as graph signals $\bbx \in \reals^n$, where $[\bbx]_i = x_i$ corresponds to the value of the data at node $i$ \cite{shuman13-mag,ortega2018graph}. Graph operations on graph signals are parametrized by the graph shift operator (GSO) $\bbS \in \reals^{n \times n}$. The GSO is a matrix that encodes the sparsity pattern of $\bbG$ by satisfying $[\bbS]_{ij} = s_{ij} \neq 0$ if and only if $i=j$ or $(i,j) \in \ccalE$. Examples of GSOs include the graph adjacency matrix $[\bbA]_{ij} = \ccalW(i,j)$ \cite{sandryhaila13-dspg}, the graph Laplacian $\bbL=\diag (\bbA\boldsymbol{1})-\bbA$ \cite{shuman13-mag}, and their normalized counterparts. In this paper, we consider $\bbS=\bbA$.

When applied to a graph signal, the GSO effectively shifts or diffuses graph data over the network edges. To see this, note that, at each node $i$, $[\bbS\bbx]_i = \sum_{j | (i,j) \in \ccalE} s_{ij}x_j$. In other words, nodes $j$ shift their data values $x_j$, weighted by the proximity measure $s_{ij}$, to neighbors $i$. Equipped with this notion of shift, we define the graph convolution as a weighted sum of data shifted to neighbors at most $K-1$ hops away. Explicitly,
\begin{equation} \label{eqn:graph_convolution}
\bbh *_{\bbS} \bbx = \sum_{k=0}^{K-1} h_k \bbS^k \bbx = \bbH(\bbS) \bbx \text{.}
\end{equation}
The vector $\bbh = [h_0, \ldots h_{K-1}]$ groups the filter coefficients and $*_{\bbS}$ denotes the convolution operation with GSO $\bbS$ \cite{moura2018convolution,segarra17-linear}.

Because $\bbS$ is symmetric, it can be diagonalized as $\bbS = \bbV \bbLam \bbV^{\Hr}$, where $\bbLam$ contains the graph eigenvalues and $\bbV$ the graph eigenvectors. The eigenvector basis $\bbV$, which we call the graph spectral basis, is orthonormal. Substituting $\bbS=\bbV\bbLam\bbV^\Hr$ in \eqref{eqn:graph_convolution} and calculating the change of basis $\bbV^\Hr \bbH(\bbS) \bbx$, we get
\begin{equation} \label{eqn:spec-lsi-gf}
\bbV^\Hr \bbH(\bbS)\bbx = \sum_{k=0}^{K-1} h_k \bbLam^k \bbV^\Hr \bbx = h(\bbLam) \bbV^\Hr \bbx \text{.}
\end{equation}
Thus, the graph convolution $\bbH(\bbS)$ has spectral representation $h(\lambda)= \sum_{k=0}^{K-1} h_k \lambda^k$, which only depends on $\bbh$ and on the eigenvalues of $\bbG$.

A GNN is made up of layers consisting of a bank of filters like the one in \eqref{eqn:graph_convolution} followed by a pointwise nonlinear activation function. Denoting the activation function $\sigma$, the $\ell$th layer of a GNN can be written as
\begin{equation} \label{eqn:gcn_layer}
\bbx^f_{\ell} = \sigma \left( \sum_{g=1}^{F_{\ell-1}} \bbh_{\ell}^{fg} *_{\bbS} \bbx_{\ell-1}^g \right)
\end{equation}
where $\bbx_{\ell}^f$ denotes the $f$th feature of layer $\ell$ for $1 \leq f \leq F_{\ell}$ and $1 \leq \ell \leq L$. At each layer $\ell$, $F_{\ell-1}$ and $F_\ell$ are the numbers of input and output features respectively. The input features of the first layer, denoted $\bbx^g_0$, are the input data $\bbx^g$ for $1 \leq g \leq F_0$. The GNN output is $\bby^f = \bbx_L^f$. For simplicity, this GNN can be succinctly represented as the map $\bby = \bbPhi(\ccalH; \bbS; \bbx)$. The set $\ccalH$ groups the learnable parameters of all layers and is defined as $\ccalH = \{\bbh_{\ell}^{fg}\}_{\ell,f,g}$.



\subsection{Graphon signal processing}

A graphon is a bounded symmetric measurable kernel $\bbW: [0,1]^2 \to [0,1]$ which can be interpreted as an undirected graph with an uncountable number of nodes. This can be seen by associating points $u_i, u_j \in [0,1]$ with graph nodes $i,j$, and assigning weights $\bbW(u_i,u_j)$ to edges $(i,j)$. 
More formally, graphons are both limit objects of convergent graph sequences and generative models for deterministic and random graphs. The first interpretation is important because it allows grouping graphs in \textit{graph families} associated with different graphons. The second provides a way of instantiating or sampling graphs belonging to a graphon family.

Consider arbitrary unweighted and undirected graphs $\bbF = (\ccalV', \ccalE')$. We refer to a given $\bbF$ as a ``graph motif''.
Homomorphisms of $\bbF$ into $\bbG = (\ccalV,\ccalE,\ccalW)$ are defined as adjacency preserving maps in which~$(i,j)\in \ccalE'$ implies $(i,j)\in \ccalE$. Note that these maps are only a fraction of the total number of ways in which the nodes of $\bbF$ can be mapped into $\bbG$. Therefore, we can define a density of homomorphisms $t(\bbF,\bbG)$, which quantifies the frequency of the motif $\bbF$ in the graph $\bbG$. Homomorphisms of graphs into graphons are defined analogously, so we can also define the density of homomorphisms of the graph $\bbF$ into the graphon $\bbW$---denoted $t(\bbF,\bbW)$. A sequence $\{\bbG_n\}$ is said to converge to the graphon $\bbW$ if, for all finite, unweighted and undirected graphs $\bbF$,
\begin{equation} \label{eqn_graphon_convergence}
   \lim_{n\to\infty} t(\bbF,\bbG_n) = t(\bbF,\bbW).
\end{equation}

It can be shown that every graphon is the limit object of a convergent graph sequence, and every convergent graph sequence converges to a graphon \cite[Chapter 11]{lovasz2012large}. Indeed, convergent graph sequences can be obtained by drawing on the generative model interpretation of graphons. The first type of graphs that can be generated from graphons are \textit{deterministic graphs} $\bar{\bbG}_n$. A deterministic graph on $n$ nodes is constructed by building a regular partition $u_i = (i-1)/n$ of $[0,1]$ for $1 \leq i \leq n$ and attaching each point $u_i$ to a node $i$. The GSO $\bar{\bbS}_n$ is then defined as 
\begin{equation} \label{eqn:deterministic_graph}
[\bar{\bbS}_n]_{ij} = \bar{s}_{ij} = \bbW(u_i,u_j)
\end{equation}
i.e., edges $(i,j)$ have weight $\bbW(u_i,u_j)$.

The second type of graphs that can be generated from a graphon are \textit{stochastic graphs} $\bbG_n$. These are unweighted graphs (i.e., whose edge weights are either $0$ or $1$) obtained by sampling edges $(i,j)$ as Bernoulli random variables with probability $[\bar{\bbS}_n]_{ij}=\bar{s}_{ij}$. Explicitly,
\begin{equation} \label{eqn:stochastic_graph}
[{\bbS}_n]_{ij} \sim \text{Bernoulli}([\bar{\bbS}_n]_{ij})\text{.}
\end{equation}
Also known as $\bbW$-random graphs, these graphs converge to the graphon $\bbW$ with probability $1$ \cite[Chapter 10.1]{lovasz2012large}. 

%

%
%

In the same way that we can instantiate graphs from graphons, the \textit{graphon induced by a graph} $\bbG_n$ can always be defined. For instance, consider the deterministic graph $\bar{\bbG}_n$ and let $I_i = [(i-1)/n,i/n]$, $1 \leq i \leq n$, be a regular partition of the unit interval. The graphon induced by $\bar{\bbG}_n$ is denoted $\bbW_{\bar{\bbG}_n} = \bar{\bbW}_n$ and defined as
\begin{equation} \label{eqn:graphon_ind}
\bar{\bbW}_{n}(u,v) = [\bar{\bbS}_n]_{ij}\ \mbI(u \in I_i) \times \mbI(v \in I_j)
\end{equation}
where $\mbI$ is the indicator function and $\times$ the Cartesian product.
The graphon $\bbW_{\bbG_n}=\bbW_n$ induced by the stochastic graph $\bbG_n$ can be defined analogously.

\subsubsection{Graphon signals and graphon convolutions}

Graphon signals are defined as functions $X \in L_2([0,1])$. They can be seen as limits of graph signals supported on graphs converging to a graphon \cite{ruiz2019graphon}, or as generating models for graph signals $\bbx_n$ supported on the deterministic graph $\bar{\bbG}_n$ \eqref{eqn:deterministic_graph} or on the stochastic graph $\bbG_n$ \eqref{eqn:stochastic_graph} and given by
\begin{equation} \label{eqn:sample_graph_signal}
[\bbx_n]_i = X(u_i)
\end{equation}
where $u_i=(i-1)/n$ for $1 \leq i \leq n$.
Graphon signals can also be induced by graph signals. The graphon signal induced by a graph signal $\bbx_n \in \reals^n$ is given by \cite{ruiz2019graphon}
\begin{equation} \label{eqn:graphon_signal_ind}
X_n(u) = [\bbx_n]_i\ \mbI(u \in I_i)
\end{equation}
where $I_i = [(i-1)/n,i/n]$, $1 \leq i \leq n$, is the regular partition of $[0,1]$.

The integral operator with kernel $\bbW$ defines a basic operation for graphon signals. Explicitly, this operator is given by
\begin{equation} \label{eqn:graphon_shift}
(T_\bbW X)(v) := \int_0^1 \bbW(u,v)X(u)du
\end{equation}
and we call it graphon shift operator (WSO) in analogy with the GSO.
Since $\bbW$ is bounded and symmetric, $T_\bbW$ is a self-adjoint Hilbert-Schmidt (HS) operator. Hence, the graphon can be expressed in the operator's spectral basis as $\bbW(u,v) = \sum_{i \in \mbZ\setminus \{0\}} \lambda_i \varphi_i(u)\varphi_i(v)$ and we can rewrite $T_\bbW$ as
\begin{equation} \label{eqn:graphon_spectra}
(T_\bbW X)(v) = \sum_{i \in \mbZ\setminus \{0\}}\lambda_i \varphi_i(v) \int_0^1 \varphi_i(u)X(u)du \text{.}
\end{equation}
The eigenvalues $\lambda_i$, $i \in \mbZ\setminus \{0\}$, are ordered as $1 \geq \lambda_1 \geq \lambda_2 \geq \ldots \geq \ldots \geq \lambda_{-2} \geq \lambda_{-1} \geq -1$ following their sign and in decreasing order of absolute value. As $|i| \to \infty$, they accumulate around 0 due to the spectral theorem for self-adjoint HS operators \cite[Theorem 3, Chapter 28]{lax02-functional}. 

Analogously to graph convolutions, graphon convolutions are defined as shift-and-sum operations in which each shift corresponds to an application of the WSO. The graphon convolution $\bbh *_{\bbW} X$ is given by 
\begin{align}\begin{split} \label{eqn:lsi-wf}
&\bbh *_{\bbW} X = \sum_{k=0}^{K-1} h_k (T_{\bbW}^{(k)} X)(v) = (T_\bbH X)(v) \quad \mbox{with} \\
&(T_{\bbW}^{(k)}X)(v) = \int_0^1 \bbW(u,v)(T_\bbW^{(k-1)} X)(u)du
\end{split}\end{align}
where $T_{\bbW}^{(0)} = \bbI$ is the identity operator and $\bbh$ is the vector of filter coefficients $\bbh = [h_0, \ldots, h_{K-1}]$ \cite{ruiz2020graphon}. Alternatively, we can use the graphon spectral decomposition to write $T_\bbH$ as
\begin{align} \label{eqn:spec-graphon_filter}
\begin{split}
(T_\bbH X)(v) &= \sum_{i \in \mbZ\setminus \{0\}} \sum_{k=0}^{K-1} h_k \lambda_i^k \varphi_i(v) \int_0^1 \varphi_i(u)X(u)du\\
&= \sum_{i \in \mbZ\setminus \{0\}} h(\lambda_i) \varphi_i(v) \int_0^1 \varphi_i(u)X(u)du \text{.}
\end{split}
\end{align}
Observe that $T_\bbH$ has spectral representation $h(\lambda) = \sum_{k=0}^{K-1} h_k \lambda^k$, which only depends on the graphon eigenvalues and the coefficients $h_k$.

\section{Graphon neural networks} \label{sec:wnns}



Leveraging the definition of the convolution operation for graphon signals, we can analogously define graphon neural networks as layered architectures consisting of banks of graphon convolutional filters and nonlinear activation functions. Denoting the nonlinear activation function $\sigma$, the $\ell$th layer of a graphon neural network can be written as \cite{ruiz2020wnn}
\begin{equation}
X^f_{\ell} = \sigma\left(\sum_{g=1}^{F_{\ell-1}} \bbh_{\ell}^{fg} *_\bbW X^g_{\ell-1} \right)
\end{equation}
where $1 \leq f \leq F_{\ell}$ and $F_{\ell}$ is the number of features at the output of layer $\ell$. For a WNN with $L$ layers, the WNN output is given by $Y^f = X_L^f$. The input features at the first layer, $X_0^g$, are the input data $X^g$ for $1 \leq g \leq F_0$.
Like the GNN, this WNN can be written more succinctly as the map $Y = \bbPhi(\ccalH; \bbW; X)$, where $\ccalH = \{\bbh_\ell^{fg}\}_{\ell,f,g}$. 

In the map $\bbPhi(\ccalH; \bbW; X)$, $\ccalH$ is independent of the graphon $\bbW$. More importantly, in the GNN $\bbPhi(\ccalH; \bbS_n; \bbx_n)$ and in the WNN $\bbPhi(\ccalH; \bbW; X)$ the sets $\ccalH$ can be the same. This allows building GNNs from WNNs by instantiating graphs $\bbG_n$ and signals $\bbx_n$ from graphons $\bbW$ and signals $X$. Similarly, we can define WNNs induced by GNNs. The WNN induced by a GNN $\bbPhi(\ccalH; \bbS_n; \bbx_n)$ is given by $\bbPhi(\ccalH; \bbW_{n}; X_n)$, where $\bbW_n$ is the {graphon induced by} $\bbG_n$ \eqref{eqn:graphon_ind} and $X_n$ is the {graphon signal induced by} $\bbx_n$ \eqref{eqn:graphon_signal_ind}.
The definition of WNNs induced by GNNs is useful because it allows comparing GNNs with WNNs. This will be leveraged to break the stability analysis of GNNs down into simpler steps, the first of which is the stability analysis of WNNs. 

\subsection{Stability of WNNs}

Let $Y=\bbPhi(\ccalH;\bbW;X)$ be a WNN on the graphon $\bbW$ and suppose that $\bbW$ is perturbed by a symmetric kernel $\bbA$ to yield the graphon $\bbW' = \bbW + \bbA$ and the WNN $Y'=\bbPhi(\ccalH;\bbW';X)$. The stability of this WNN can be quantified by bounding $\|Y-Y'\|$. We start by introducing the constants $n_c^{(p)}$ and $\delta_c^{(pq)}$.

\begin{definition} \label{def:constants}
Let $\bbW^{(p)}$ and $\bbW^{(q)}$ be two graphons with labels $(p)$ and $(q)$ assigned arbitrarily. Let $c \in [0,1]$. We define $n^{(p)}_c$ and $\delta^{(pq)}_c$ as
\begin{align}
\begin{split}
n^{(p)}_c &= |\ccalC^{(p)}| \quad \text{with}\quad \ccalC^{(p)} = \{i\ |\ |\lambda^{(p)}_{i}| \geq c \}\\
\delta^{(pq)}_c &= \min_{i \in \ccalC^{(q)}}\{|\lambda^{(p)}_{i}-{\lambda}^{(q)}_{i+\mbox{\scriptsize sgn}(i)}|, |{\lambda}^{(q)}_{i}-\lambda^{(p)}_{i+\mbox{\scriptsize sgn}(i)}|, \\
&\quad \quad \quad \quad \quad \quad \quad \quad \quad  |\lambda^{(p)}_{1}-{\lambda}^{(q)}_{-1}|, |{\lambda}^{(q)}_{1}-\lambda^{(p)}_{-1}|\}
\end{split}
\end{align}
where $\lambda_i^{(p)}$ and $\lambda_i^{(q)}$ are the eigenvalues of $\bbW^{(p)}$ and $\bbW^{(q)}$ respectively.
\end{definition}
For a graphon labeled $\bbW^{(p)}$, $n_c^{(p)}$ corresponds to the number of eigenvalues with indices in the set $\ccalC^{(p)}$, i.e., to the number of eigenvalues satisfying $\lambda_i^{(p)}\geq c$. For any two graphons $\bbW^{(p)}$ and $\bbW^{(q)}$, $\delta_c^{(pq)}$ corresponds to the minimum eigenvalue difference (or eigengap) between eigenvalues $\lambda_i^{(q)}$ with $i \in \ccalC^{(q)}$ and $\lambda_j^{(p)}$ with $j$ consecutive to $i$.

The following assumptions are needed to prove stability of WNNs in Theorem \ref{thm:wnn_stab}, whose proof we defer to the appendices \cite{ruiz2021extended}.


\begin{assumption} \label{as:filter_lipschitz}
The convolutional filters $h(\lambda)$ are $A_2$-Lipschitz and non-amplifying, i.e., $|h(\lambda)-h(\mu)| \leq A_2 |\lambda-\mu|$ and $|h(\lambda)| < 1$.
\end{assumption}

\begin{assumption} \label{as:nonlinearity_lipschitz}
The activation functions are normalized Lipschitz, i.e., $|\sigma(x_2)-\sigma(x_1)|\leq |x_2-x_1|$, and $\sigma(0)=0$.
\end{assumption}

Note that AS\ref{as:filter_lipschitz} can be satisfied by adding penalties during training, and AS\ref{as:nonlinearity_lipschitz} is satisfied for most common activation functions.


\begin{theorem}[WNN stability] \label{thm:wnn_stab}
Let $\bbW: [0,1]^2 \to [0,1]$ be a graphon and let $\bbW': [0,1]^2\to[0,1]$ be a perturbation of this graphon such that $\bbW'-\bbW=\bbA$ where $\bbA$ is symmetric and $\|\bbA\| \leq \varepsilon$. Consider the $L$-layer WNNs given by $Y=\bbPhi(\ccalH;\bbW;X)$ and $Y'=\bbPhi(\ccalH;\bbW';X)$ where $F_0=F_L=1$ and $F_\ell = F$ for $1 \leq \ell \leq L-1$. Let the graphon convolutions \eqref{eqn:lsi-wf} in all layers be such that  $h(\lambda)$ is constant for $|\lambda|<c$. Then, under Assumptions \ref{as:filter_lipschitz}--\ref{as:nonlinearity_lipschitz}, it holds
\begin{equation*}
\|Y' - Y\| \leq LF^{L-1}\left(A_2 + \dfrac{\pi n_c}{\delta_c}\right)\varepsilon \|X\|
\end{equation*}
where we have labeled the graphons $\bbW^{(1)}:=\bbW$ and $\bbW^{(2)}:=\bbW'$ and defined $n_c := n_c^{(2)}$ and $\delta_c := \delta_c^{(12)}$ [cf. Definition \ref{def:constants}]. 
\end{theorem}
WNNs are thus Lipschitz stable to absolute graphon perturbations with a stability constant that depends on the hyperparameters of the WNN and on the passing band of the graphon convolutions. The dependence on the architecture is exerted by the number of layers $L$, the number of features $F$, and the convolutional filters $h(\lambda)$ through the Lipschitz constants $A_2$, $n_c$ and $\delta_c$. The constant $A_2$ measures the variability of the filter. The values of $n_c$ and $\delta_c$ depend on the spectra of both $\bbW$ and $\bbW'$, but are primarily related to the length of the filter passing band $1-c$. Since graphon eigenvalues accumulate around zero for $|i|\to\infty$, small values of $c$ translate into large $n_c$, i.e., large number of eigenvalues within the passing band. They also produce small $\delta_c$, i.e., small minimum eigengap between these eigenvalues. Hence, the larger the passing band, the worse  WNN stability, pointing to a trade-off between stability and discriminability in WNNs.


\section{Graph Neural Network Stability}
\label{sec:stability}


Theorem \ref{thm:wnn_stab} can be combined with the GNN-WNN approximation theorem from \cite[Theorem 1]{ruiz2020wnn} to prove stability of GNNs supported on deterministic graphs instantiated from $\bbW$ and $\bbW'$. This theorem bounds the error incurred when using a GNN $\bbPhi(\ccalH; \bar{\bbS}_n; \bbx_n)$ to approximate the WNN $\bbPhi(\ccalH; \bbW; X)$, where $\bar{\bbS}_n$ and $\bbx_n$ are obtained from $\bbW$ and $X$ as in \eqref{eqn:deterministic_graph} and \eqref{eqn:sample_graph_signal}. Stability of GNNs $\bbPhi(\ccalH; \bar{\bbS}_n; \bbx_n)$ to graphon perturbations is stated and proved in Theorem \ref{thm:det_graph_stab}, under the same assumptions of Theorem \ref{thm:wnn_stab} and Assumptions \ref{as:graphon_lipschitz} and \ref{as:perturb_lipschitz} below.

\begin{assumption} \label{as:graphon_lipschitz}
The graphon $\bbW$ is $A_1$-Lipschitz, i.e., $|\bbW(x_2,y_2)-\bbW(x_1,y_1)|\leq A_1(|x_2-x_1|+|y_2-y_1|)$.
\end{assumption}

\begin{assumption} \label{as:perturb_lipschitz}
The perturbation $\bbA=\bbW'-\bbW$ is $A_3$-Lipschitz.
\end{assumption}


\begin{theorem} \label{thm:det_graph_stab}
Let $\bar{\bbG}_n$ and $\bar{\bbG}_n'$ be $n$-node \textit{deterministic graphs} \eqref{eqn:sample_graph_signal} obtained from the graphons $\bbW$ and $\bbW'=\bbW+\bbA$, where $\bbA$ is symmetric and $\|\bbA\| \leq \varepsilon$. Let $\bar{\bbS}_n, \bar{\bbS}_n' \in \reals^{n\times n}$ denote their GSOs. Consider the $L$-layer GNNs given by $\bar{\bby}_n=\bbPhi(\ccalH;\bar{\bbS}_n;\bbx_n)$ and $\bar{\bby}_n'=\bbPhi(\ccalH;\bar{\bbS}_n';\bbx_n)$ where $F_0=F_L=1$ and $F_\ell = F$ for $1 \leq \ell \leq L-1$. Let the graph convolutions \eqref{eqn:graph_convolution} in all layers be such that $h(\lambda/n)$ is constant for $|\lambda/n|<c$. Then, under Assumptions \ref{as:filter_lipschitz}--\ref{as:perturb_lipschitz}, it holds
\begin{equation*}
\|\bar{\bby}_n'-\bar{\bby}_n\| \leq LF^{L-1}\left(A_2 + \frac{\pi \hat{n}_c}{\hat{\delta}_c}\right) \left( \varepsilon +\dfrac{B}{\sqrt{n}}\right)\|\bbx_n\| 
\end{equation*}
where $B=\sqrt{A_1}+\sqrt{A_1+A_3}$ and $\hat{n}_c$ and $\hat{\delta}_c$ are defined as $\hat{n}_c := \max_{p\in \{2,3,4\}}{n_c^{(p)}}$ and $\hat{\delta}_c := \min\{\delta_c^{(12)}, \delta_c^{(13)},\delta_c^{(24)}\}$ for $\bbW^{(1)}:=\bbW$, $\bbW^{(2)}:=\bbW'$, $\bbW^{(3)}:=\bar{\bbW}_n$ and $\bbW^{(4)}:=\bar{\bbW}_n'$ [cf. Definition \ref{def:constants}].
\end{theorem}
\begin{proof}
Leveraging the identities $\|\bar{\bby}_n'-\bar{\bby}_n\|_2 = \sqrt{n}\|\bar{Y}_n'-\bar{Y}_n\|_{L^2([0,1])}$ and $\|\bbx_n\|_2 = \sqrt{n}\|X_n\|_{L^2([0,1])}$, we reason in terms of the induced graphon signals $\bar{Y}_n'$ and $\bar{Y}_n$ \eqref{eqn:graphon_signal_ind} to bound $\|\bar{Y}_n'-\bar{Y}_n\|_{L^2([0,1])}$ as
\begin{align}
\begin{split}
\|\bar{Y}_n'-\bar{Y}_n\| &= \|\bar{Y}_n'-Y'+Y'-Y+Y-\bar{Y}_n\| \\
&\leq \|\bar{Y}_n'-Y'\|+\|Y'-Y\|+\|Y-\bar{Y}_n\|
\end{split}
\end{align}
using the triangle inequality.
The first and third terms on the RHS of this expression are bounded by \cite[Theorem 1]{ruiz2020wnn}. The second term is bounded by Theorem \ref{thm:wnn_stab}. Taking the maximum among the $n_c^{(p)}$ and the minimum among the $\delta_c^{(pq)}$ completes the proof.
\end{proof}

Theorem \ref{thm:det_graph_stab} can be further extended to GNNs $\bbPhi(\ccalH; \bbS_n; \bbx_n)$ and $\bbPhi(\ccalH; \bbS_n'; \bbx_n)$ defined on the stochastic graphs $\bbG_n$ and $\bbG_n'$ \eqref{eqn:stochastic_graph}.


\begin{assumption} \label{as:graphon_degree}
For a fixed value of $\xi \in (0,1)$, $n$ is such that
\begin{align}
n-\dfrac{\log{{2n}/{\xi}}}{d_\bbW} > 2\dfrac{A_1}{d_\bbW} \quad \text{and} \quad 
n-\dfrac{\log{{2n}/{\xi}}}{d_{\bbW'}} > 2\dfrac{A_1+A_3}{d_{\bbW'}}
\end{align}
where $d_\bbW, d_{\bbW'}$ denote the maximum degree of the graphons $\bbW, \bbW'$, i.e., $d_\bbW = \max_x \int_0^1 \bbW(x,y) dy$.
\end{assumption}


\begin{figure}[t]
\begin{minipage}[b]{1.0\linewidth}
  \centering
  \centerline{\includegraphics[height=0.2\textheight]{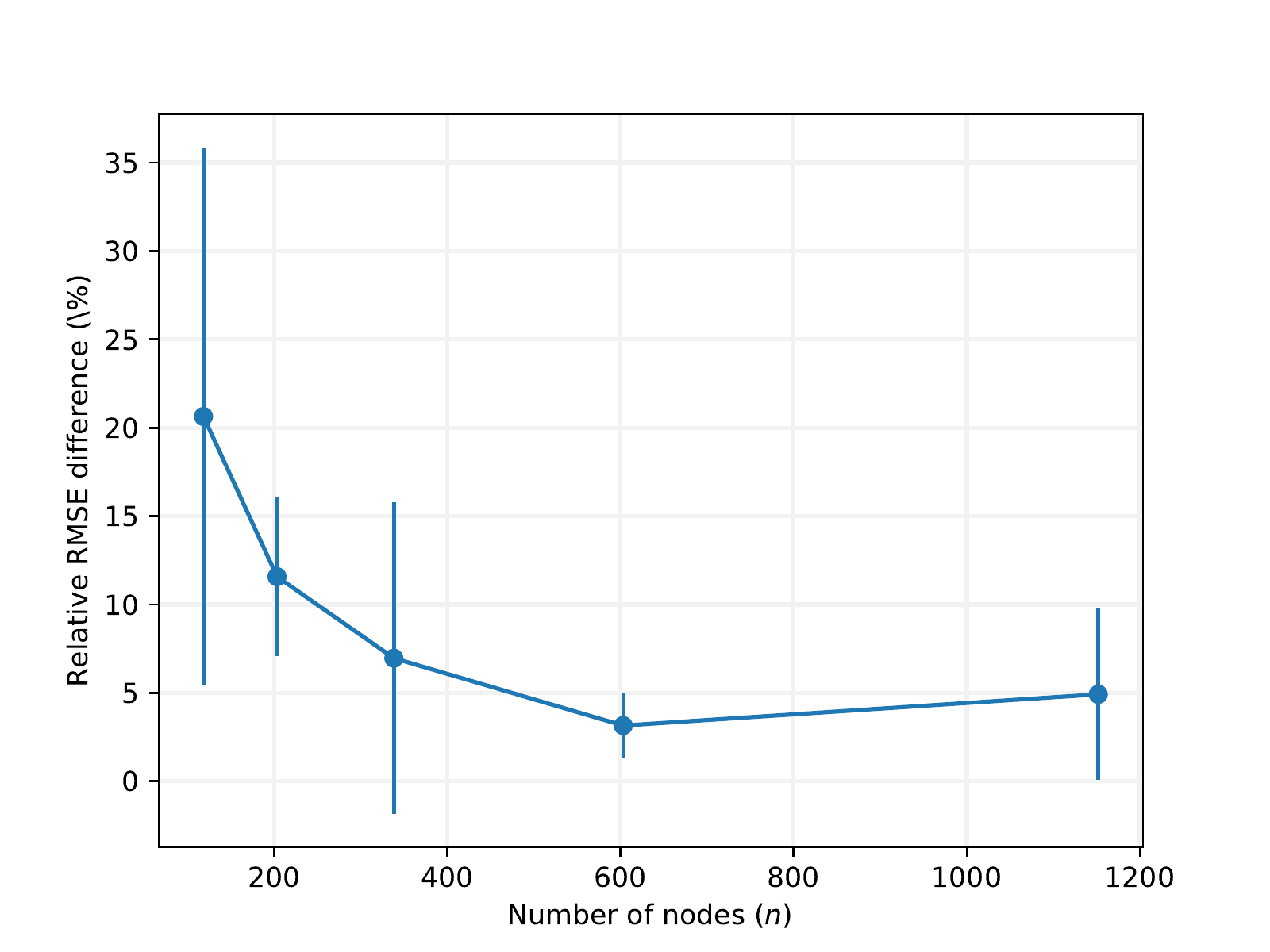}}
\end{minipage}
\caption{Relative RMSE difference on the test set between graphs $\bbG_n$ and $\bbG_n'$ for the movie ``Contact''.}
\label{fig:res}
\end{figure}


\begin{theorem} \label{thm:stab_stochastic}
Let $\bbG_n$ and $\bbG_n'$ be $n$-node \textit{stochastic graphs} \eqref{eqn:stochastic_graph} obtained from the graphons $\bbW$ and $\bbW'=\bbW+\bbA$, where $\bbA$ is symmetric and $\|\bbA\| \leq \varepsilon$. Let $\bbS_n, \bbS_n' \in \reals^{n\times n}$ denote their GSOs. Consider the $L$-layer GNNs given by $\bby_n=\bbPhi(\ccalH;\bbS_n;\bbx_n)$ and $\bby_n'=\bbPhi(\ccalH;\bbS_n';\bbx_n)$ where $F_0=F_L=1$ and $F_\ell = F$ for $1 \leq \ell \leq L-1$. Let the graph convolutions \eqref{eqn:graph_convolution} in all layers be such that $h(\lambda/n)$ is constant for $|\lambda/n|<c$. Then, under Assumptions \ref{as:filter_lipschitz}--\ref{as:graphon_degree}, it holds with probability at least $1-\xi$
\begin{align*}
\begin{split}
\|\bby_n'-\bby_n\| \leq LF^{L-1}\left(A_2 + \frac{\pi \check{n}_c}{\check{\delta}_c}\right) \left( \varepsilon +\dfrac{B+4\sqrt{\log{2n/\xi}}}{\sqrt{n}}\right)\|\bbx_n\| 
\end{split}
\end{align*}
where $B=\sqrt{A_1}+\sqrt{A_1+A_3}$ and $\check{n}_c$ and $\check{\delta}_c$ are defined as $\check{n}_c := \max\{\max_{p\in \{5,6\}}{n_c^{(p)}},\hat{n}_c\}$ and $\check{\delta}_c := \min\{\delta_c^{(35)},\delta_c^{(46)},\hat{\delta}_c\}$ for $\bbW^{(3)}:=\bar{\bbW}_n$, $\bbW^{(4)}:=\bar{\bbW}_n'$, $\bbW^{(5)}:={\bbW}_n$, $\bbW^{(6)}:={\bbW}_n'$ [cf. Definition \ref{def:constants}] and $\hat{n}_c$, $\hat{\delta}_c$ as in Theorem \ref{thm:det_graph_stab}.
\end{theorem}
\begin{proof}[Proof of Theorem \ref{thm:stab_stochastic}]
The proof follows trivially from Theorem \ref{thm:det_graph_stab} and the following lemma, whose proof we defer to the appendices \cite{ruiz2021extended}.
\begin{lemma} \label{lemma_random_graph}
Let $\bar{\bbG}_n$ be an undirected graph with $\bar{\bbS}_n \in [0,1]^{n \times n}$. Let $\bbG_n$ be a stochastic graph obtained from $\bar{\bbG}_n$ as in \eqref{eqn:stochastic_graph}. Consider the $L$-layer GNNs given by $\bar{\bby}_n=\bbPhi(\ccalH;\bar{\bbS}_n;\bbx_n)$ and ${\bby}_n=\bbPhi(\ccalH;{\bbS}_n;\bbx_n)$ where $F_0=F_L=1$ and $F_\ell = F$ for $1 \leq \ell \leq L-1$. Let the graph convolutions \eqref{eqn:graph_convolution} in all layers be such that $h(\lambda/n)$ is constant for $|\lambda/n|<c$. Then, under Assumptions \ref{as:filter_lipschitz},\ref{as:nonlinearity_lipschitz} and \ref{as:graphon_degree}, it holds with probability at least $1-\xi$
\begin{align*}
\begin{split}
\|\bby_n-\bar{\bby}_n\| \leq LF^{L-1}\left(A_2 + \frac{\pi {n}^{(q)}_c}{{\delta}^{(pq)}_c}\right) \left({2\sqrt{\log{2n/\xi}}}/{\sqrt{n}}\right)\|\bbx_n\| 
\end{split}
\end{align*}
where ${n}^{(q)}_c$ and ${\delta}^{(pq)}_c$ are as in Definition \ref{def:constants} for $\bbW^{(p)}=\bar{\bbW}_n$, the graphon induced by $\bar{\bbG}_n$, and $\bbW^{(q)}={\bbW}_n$, the graphon induced by ${\bbG}_n$ \eqref{eqn:graphon_ind}. 
\end{lemma}
Use the triangle inequality to write
\begin{align}
\begin{split}
\|\bby_n-\bby_n'\| &= \|\bby_n-\bar{\bby}_n+\bar{\bby}_n-\bar{\bby}_n'+\bar{\bby}_n'-\bby_n'\| \\
&\leq \|\bby_n-\bar{\bby}_n\|+\|\bar{\bby}_n-\bar{\bby}_n'\|+\|\bar{\bby}_n'-\bby_n'\|\text{.}
\end{split}
\end{align}
The first and third terms on the RHS of this expression are bounded by Lemma \ref{lemma_random_graph}. The second term is bounded by Theorem \ref{thm:det_graph_stab}. Taking the maximum of the $n_c^{(p)}$ and the minimum of the $\delta_c^{(pq)}$ completes the proof.

\end{proof}

From Theorems \ref{thm:det_graph_stab} and \ref{thm:stab_stochastic}, we conclude that GNNs are asymptotically stable to absolute graphon perturbations on both deterministic and stochastic graphs associated with a graphon. This is an important result, as absolute graphon perturbations can be used to model a variety of graph perturbations. 
We further observe that the stability bound depends on the graphon variability $A_1$ and on the perturbation variability $A_3$. It is also inversely proportional to the size of the graph. In particular, as $n\to\infty$ the GNN stability bound converges to the stability bound of the WNN [cf. Theorem \ref{thm:wnn_stab}]. {We demonstrate this asymptotic behavior in the numerical experiment of Section \ref{sec:sims}}.

Like in WNNs, in GNNs stability decreases with the architecture depth $L$ and width $F$. It also depends on the variability $A_2$ and on the passing band $1-c$ of the convolutional filters $h(\lambda)$. The wider the passing band, the greater the number of eigenvalues $\hat{n}_c$, $\check{n}_c$ larger than $c$, and the smaller the minimum eigengap $\hat{\delta}_c$, $\check{\delta}_c$. Therefore, GNNs also exhibit a trade-off between stability and discriminability.
In practice, however, this issue is alleviated by the nonlinearities $\sigma$, which act as rectifiers that scatter spectral components associated with small eigenvalues to upper parts of the spectrum where they can be discriminated.





\section{Numerical Experiments}
\label{sec:sims}

Using the MovieLens 100k dataset \cite{harper16-movielens}, we implement a GNN supported on a movie similarity network to predict the ratings that different users would give to the movie ``Contact'' considering their ratings to other movies in the network and the similarities between these movies. The dataset consists of 100,000 integer ratings between 1 and 5 given by $U=943$ users to $M=1682$ movies. To build the movie similarity graph, we first store the data in a $U \times M$ matrix $\bbR$ such that $[\bbR]_{um}$ is the rating of user $u$ to movie $m$ if it exists, and zero otherwise; then, we compute the edge weights by calculating pairwise correlations between different movies, i.e., betweeen the columns of $\bbR$. Denoting the index of the movie ``Contact'' $m_0$, the input samples $\bbx_u$ are generated by zero-ing out the entry $[\bbR]_{um_0}$ of the user vectors $\bbr_u=[\bbR]_{u:}$. The output samples are vectors $\bby_u$ where $[\bby_u]_{m}=[\bbR]_{um_0}$ if $m=m_0$ and zero otherwise. 

To analyze the stability of GNNs to graphon perturbations in this setting, we consider the full $1682$-movie network to be a SBM graphon $\bbW$, and perturb it to obtain $\bbW'=\bbW+\bbA$ where $\bbA(u,v) = (1 - \exp{(\bbW(u,v)^{-1})})/10$. We then instantiate deterministic graphs $\bar{\bbG}_n$ and $\bar{\bbG}_n'$ with $n=118, 203, 338, 603, 1152$, and use a 90-10 data split to train a GNN with $L=2$, $F=64$ and $K=5$ on the $\bar{\bbG}_n$ by optimizing the MSE loss. The trained GNNs are tested on both the $\bar{\bbG}_n$ and $\bar{\bbG}_n'$; the relative RMSE difference obtained on the original and perturbed graphs is shown in Fig. \ref{fig:res}. We observe the asymptotic behavior predicted by Theorem \ref{thm:det_graph_stab}---the larger the graph, the more stable the GNN.


\section{Conclusions}
\label{sec:conclusions}

In this paper, we have defined graphon neural networks (WNNs) and characterized their stability to perturbations of the graphon. By interpreting WNNs as generating models for GNNs, we have further analyzed stability of GNNs defined on deterministic and stochastic graphs instantiated from the original and perturbed graphons. These are key contributions to the study of the theoretical properties of GNNs as graphon perturbations, being perturbations of graph models, make for a more general type of graph perturbation than matrix perturbations. We conclude that GNNs are stable to graphon perturbations with a stability bound that decreases asymptotically with the size of the graph. This result has been demonstrated numerically in a movie recommendation experiment.
\appendix
\section{Proof of Theorem 1}
We start by proving stability of graphon convolutions to absolute graphon perturbations, which is stated in Theorem \ref{thm:wft_stability}.
\begin{theorem}\label{thm:wft_stability}
Consider the graphon convolution given by $Y=T_\bbH X$ as defined in \eqref{eqn:spec-graphon_filter}, where $h(\lambda)$ is constant for $|\lambda|<c$. Let $Y'=T_{\bbH'} X$ be the perturbed graphon convolution. Under Assumption \ref{as:filter_lipschitz}, it holds that:
\begin{equation}
\label{eqn:thm4}
\|Y-Y'\|\leq \left(A_2 +\frac{\pi n_c}{\delta_c} \right) \epsilon \|X\|,
\end{equation}
where we have labeled the graphons $\bbW^{(1)}:=\bbW$ and $\bbW^{(2)}:=\bbW'$, and defined $n_c := n_c^{(2)}$ and $\delta_c := \delta_c^{(12)}$ [cf. Definition \ref{def:constants}]. 
\end{theorem}
To prove this theorem, we need the following propositions.

\begin{proposition}\label{thm:davis_kahan}
Let $T$ and $T^\prime$ be two self-adjoint operators on a separable Hilbert space $\ccalH$ whose spectra are partitioned as $\gamma \cup \Gamma$ and $\omega \cup \Omega$ respectively, with $\gamma \cap \Gamma = \emptyset$ and $\omega \cap \Omega = \emptyset$. If there exists $d > 0$ such that $\min_{x \in \gamma,\, y \in \Omega} |{x - y}| \geq d$ and $\min_{x \in \omega,\, y \in \Gamma}|{x - y}| \geq d$, then
\begin{equation*}\label{eqn:davis_kahan}
	\|E_T(\gamma) - E_{T^\prime}(\omega)\| \leq \frac{\pi}{2} \frac{\|{T - T^\prime}\|}{d}
\end{equation*}
\end{proposition}
\begin{proof}
See \cite{seelmann2014notes}.
\end{proof}

\begin{proposition} \label{prop:eigenvalue_diff}
Let $\bbW:[0,1]^2\to[0,1]$ and $\bbW':[0,1]^2\to[0,1]$ be two graphons with eigenvalues given by $\{\lambda_i(T_\bbW)\}_{i\in\mbZ\setminus\{0\}}$ and $\{\lambda_i(T_{\bbW'})\}_{i\in\mbZ\setminus\{0\}}$, ordered according to their sign and in decreasing order of absolute value. Then, for all $i \in \mbZ \setminus \{0\}$, the following inequalities hold
\begin{equation*}
|\lambda_i(T_{\bbW'})-\lambda_i(T_\bbW)| \leq \|T_{\bbW'-\bbW}\| \leq \|\bbW'-\bbW\|\ .
\end{equation*}
\end{proposition}
\begin{proof}
See \cite{ruiz2020wnn}.
\end{proof}

We first prove Theorem \ref{thm:wft_stability} for filters satisfying $h(\lambda)=0$ for $|\lambda| <c$. Writing the inner product $\int_0^1 X(u) \varphi_i(u) du$ as $\hat{X}(\lambda_i)$, the filter output norm difference can be expressed as
\begin{align*}
\begin{split}
\left\|T_\bbH X-T_{\bbH'} X\right\| &= \left\|\sum_{i} h(\lambda_i)\hat{X}(\lambda_i) \varphi_i - \sum_i h(\lambda'_i)\hat{X}(\lambda'_i)\varphi'_i\right\| \\
& = \left\|\sum_{i} h(\lambda_i)\hat{X}(\lambda_i) \varphi_i -  h(\lambda'_i)\hat{X}(\lambda'_i)\varphi'_i\right\|. 
\end{split}
\end{align*}
Adding and subtracting $h(\lambda'_i)\hat{X}(\lambda_i)\varphi_i $ and applying the triangle inequality, we get
\begin{align*}
\begin{split}
\left\| T_\bbH X - T_{\bbH'} X\right\|  &= \left\|\sum_{i} h(\lambda_i)\hat{X}(\lambda_i) \varphi_i -  h(\lambda'_i)\hat{X}(\lambda'_i)\varphi'_i\right\|\\
&= \Bigg\|\sum_{i} h(\lambda_i)\hat{X}(\lambda_i) \varphi_i - h(\lambda'_i)\hat{X}(\lambda_i)\varphi_i 
\\& \quad+ h(\lambda'_i)\hat{X}(\lambda_i)\varphi_i - h(\lambda'_i)\hat{X}(\lambda'_i)\varphi'_i \Bigg\| \\
&\leq \left\|\sum_{i} \left(h(\lambda_i)-h(\lambda'_i)\right)\hat{X}(\lambda_i) \varphi_i \right\|
\mbox{  \textbf{(1.1)}} \\
&+ \left\|\sum_{i} h(\lambda'_i) \left(\hat{X}(\lambda_i) \varphi_i - \hat{X}(\lambda'_i)\varphi'_i \right) \right\| \mbox{  \textbf{(1.2)}}
\end{split}
\end{align*}
which we have split between \textbf{(1.1)} and \textbf{(1.2)}. 

We start by bounding \textbf{1.1}. To do so, we leverage the Lipschitz continuous assumption, which can be expressed as $| h(\lambda_i)-h(\lambda'_i) |\leq A_2\|\lambda_i - \lambda'_i |$. Using also Prop. \ref{prop:eigenvalue_diff} and the Cauchy-Schwarz inequality, we have
\begin{align*}
\left\|\sum_{i} \left(h(\lambda_i)-h(\lambda'_i)\right)\hat{X}(\lambda_i) \varphi_i \right\|\leq A_2\| \bbW - \bbW' \| \left\|\sum_{i} \hat{X}(\lambda_i) \varphi_i\right\|
\end{align*}
and, since $\|\bbW-\bbW'\|\leq \|\bbA\|=\epsilon$,
\begin{align}\label{eqn:thm4.1}
\left\|\sum_{i} \left(h(\lambda_i)-h(\lambda'_i)\right)\hat{X}(\lambda_i) \varphi_i \right\|\leq A_2\epsilon \| X \|\end{align}

For  \textbf{(1.2)}, we use the triangle and Cauchy-Schwarz inequalities to obtain
\begin{align*}
\begin{split}
&\left\|\sum_{i} h(\lambda'_i) \left(\hat{X}(\lambda_i) \varphi_i - \hat{X}(\lambda'_i)\varphi'_i \right) \right\| 
\\&= \left\|\sum_{i} h(\lambda'_i) \left(\hat{X}(\lambda_i) \varphi_i + \hat{X}(\lambda_i) \varphi'_i - \hat{X}(\lambda_i) \varphi'_i - \hat{X}(\lambda'_i)\varphi'_i \right) \right\| \\
&\leq \left\|\sum_{i} h(\lambda'_i) \hat{X}(\lambda_i) (\varphi_i - \varphi'_i) \right\| + \left\|\sum_{i} h(\lambda'_i) \varphi'_i \langle X, \varphi_i - \varphi'_i \rangle \right\| \\
&\leq 2\sum_{i} \|h(\lambda'_i)\| \|X\| \|\varphi_i - \varphi'_i\|.
\end{split}
\end{align*}

The norm difference of the eigenfunctions is bounded by Prop. \ref{thm:davis_kahan}. Thus, we can write
\begin{align*}
\begin{split}
\left\|\sum_{i} h(\lambda'_i) \left(\hat{X}(\lambda_i) \varphi_i - \hat{X}(\lambda'_i)\varphi'_i \right) \right\| &\leq \|X\| \sum_{i} \|h(\lambda'_i)\| \frac{\pi\|T_\bbW - T_{\bbW'}\|}{d_i}
\end{split}
\end{align*}
where $d_i$ is the minimum between $\min(|\lambda_i - \lambda'_{i+1}|,|\lambda_i-\lambda'_{i-1}|)$ and $\min(|\lambda'_i - \lambda_{i+1}|,|\lambda'_i-\lambda_{i-1}|)$ for each $i$. By Definition \ref{def:constants}, we have $\delta_c\leq d_i$ for all $i$ and $\|T_\bbW - T_{\bbW'}\| \leq \|\bbW-\bbW'\|$ (i.e., the Hilbert-Schmidt norm dominates the operator norm). Therefore, we obtain
\begin{align*}
\begin{split}
\left\|\sum_{i} h(\lambda'_i) \left(\hat{X}(\lambda_i) \varphi_i - \hat{X}(\lambda'_i)\varphi'_i \right) \right\| &\leq \frac{\pi\|\bbW - {\bbW'}\|}{\delta_c} \|X\| \sum_{i} \|h(\lambda'_i)\|\ .
\end{split}
\end{align*}
Note that, since $|h(\lambda)|<1$ and $h(\lambda)=0$ for $|\lambda|<c$, the sum on the right hand side is finite. Denoting the number of eigenvalues $|\lambda'|\geq c$ as $n_c$, this inequality can be rewritten as
\begin{align} \label{eqn:thm4.2}
\begin{split}
\left\|\sum_{i} h(\lambda'_i) \left(\hat{X}(\lambda_i) \varphi_i - \hat{X}(\lambda'_i)\varphi'_i \right) \right\| &\leq \frac{\pi\epsilon} {\delta_c } \|X\|  n_c.
\end{split}
\end{align} 
The result stated in Theorem \ref{thm:wft_stability} follows from \eqref{eqn:thm4.1} and \eqref{eqn:thm4.2}. 

In Theorem \ref{thm:wft_stability}, the filter $h(\lambda)$ is constant for $|\lambda|<c$, but not necessarily zero. This type of filter can be constructed by adding a low-pass filter with band $|\lambda|\geq c$ and a high-pass filter $g(\lambda)$ which is constant for $|\lambda|<c$. The high-pass filter is stable because it can be seen as the sum of a constant filter with gain $G$ and a low-pass filter with negative gain $-G$. 

We now extend the stability result to the WNN. To compute a bound for $\|Y-Y'\|$, we start by writing it in terms of the features of the final layer, i.e.
\begin{equation} \label{eqn:proof2.0}
\left\|Y-Y'\right\|^2 = \sum_{f=1}^{F_L} \left\|X^f_{L} - X^{'f}_{L}\right\|^2.
\end{equation}
At layer $\ell$ of the WNN $\bbPhi(\ccalH;\bbW;X)$, we have
\begin{align*}
\begin{split}
X^f_{\ell} = \sigma\left(\sum_{g=1}^{F_{\ell-1}}\bbh_\ell^{fg} *_{\bbW} X_{\ell-1}^g\right) = \sigma\left(\sum_{g=1}^{F_{\ell-1}}T_{\bbH_\ell^{fg}} X_{\ell-1}^g\right)
\end{split}
\end{align*}
and, similarly for $\bbPhi(\ccalH;\bbW';X)$,
\begin{align*}
\begin{split}
X^f_{\ell} = \sigma\left(\sum_{g=1}^{F_{\ell-1}}\bbh^{'fg}_{\ell} *_{\bbW} X^{'g}_{\ell-1}\right) = \sigma\left(\sum_{g=1}^{F_{\ell-1}}T_{\bbH^{'fg}_{\ell}} X^{'g}_{\ell-1}\right).
\end{split}
\end{align*}
We can therefore write $\|X^f_{\ell}-X^{'f}_{\ell}\|$ as
\begin{align*}
\begin{split}
\left\|X^f_{\ell}-X^{'f}_{\ell}\right\| = \left\|\sigma\left(\sum_{g=1}^{F_{\ell-1}}T_{\bbH_\ell^{fg}} X_{\ell-1}^g\right)-\sigma\left(\sum_{g=1}^{F_{\ell-1}}T_{\bbH^{'fg}_{\ell}} X^{'g}_{\ell-1}\right)\right\|
\end{split}
\end{align*}
and, since $\sigma$ is normalized Lipschitz,
\begin{align*}
\begin{split}
\left\|X^f_{\ell}-X^{'f}_{\ell}\right\|&\leq \left\|\sum_{g=1}^{F_{\ell-1}}T_{\bbH_\ell^{fg}} X_{\ell-1}^g-T_{\bbH^{'fg}_{\ell}} X^{'g}_{\ell-1}\right\| \\
&\leq \sum_{g=1}^{F_{\ell-1}}\left\|T_{\bbH_\ell^{fg}} X_{\ell-1}^g-T_{\bbH^{'fg}_{\ell}} X^{'g}_{\ell-1}\right\|.
\end{split}
\end{align*}
where the second inequality follows from the triangle inequality. Looking at each feature $g$ independently, we apply the triangle inequality once again to get
\begin{align*}
\begin{split}
&\left\|T_{\bbH_\ell^{fg}} X_{\ell-1}^g-T_{\bbH^{'fg}_{\ell}} X^{'g}_{\ell-1}\right\|
\\& \leq \left\|T_{\bbH_\ell^{fg}} X^g_{\ell-1}-T_{\bbH^{'fg}_{\ell}} X^g_{\ell-1}\right\|_{L_2} + \left\|T_{\bbH^{'fg}_{\ell}}\left(X^g_{\ell-1}-X^{'g}_{\ell-1}\right)\right\|.
\end{split}
\end{align*}
The first term on the right hand side of this inequality is bounded by \eqref{eqn:thm4} in Theorem \ref{thm:wft_stability}. The second term can be decomposed by using Cauchy-Schwarz and recalling that $|h(\lambda)| < 1$ for all graphon convolutions in the WNN. We thus obtain a recursion for $\|X^f_{\ell}-X^{'f}_{\ell}\|$, which is given by
\begin{align} \label{eqn:proof2.1}
\begin{split}
\left\|X^f_{\ell}-X^{'f}_{\ell}\right\| \leq \sum_{g=1}^{F_{\ell-1}} \left(A_2 + \frac{\pi n_c}{\delta_c}\right)\epsilon \|X^g_{\ell-1}\|  + \sum_{g=1}^{F_{\ell-1}}\left\|X^g_{\ell-1}-X^{'g}_{\ell-1}\right\|
\end{split}
\end{align}
and whose first term, $\sum_{g=1}^{F_0} \|X_0^g-X^{'g}_{0}\|$, is equal to $0$. To solve this recursion, we need to compute the norm $\|X_{\ell-1}^g\|$. Since the nonlinearity $\sigma$ is normalized Lipschitz and $\sigma(0)=0$ by Assumption \ref{as:filter_lipschitz}, this bound can be written as
\begin{align*}
\begin{split}
\left\|X_{\ell-1}^g\right\| \leq \left\|\sum_{g=1}^{F_{\ell-1}} T_{\bbH_\ell^{fg}} X_{\ell-1}^g \right\|
\end{split}
\end{align*}
and using the triangle and Cauchy Schwarz inequalities,
\begin{align*}
\begin{split}
\left\|X_{\ell-1}^g\right\| \leq \sum_{g=1}^{F_{\ell-1}}\left\|T_{\bbH_\ell^{fg}} \right\| \left\|X_{\ell-1}^g \right\| \leq \sum_{g=1}^{F_{\ell-1}}\left\|X_{\ell-1}^g \right\|
\end{split}
\end{align*}
where the second inequality follows from $|h(\lambda)| < 1$. Expanding this expression with initial condition $X_0^g = X^g$ yields
\begin{align} \label{eqn:proof2.2}
\begin{split}
\left\|X_{\ell-1}^g\right\| \leq \prod_{\ell'=1}^{\ell-1} F_{\ell'} \sum_{g=1}^{F_{0}}\left\|X^g \right\|.
\end{split}
\end{align}
and substituting it back in \eqref{eqn:proof2.1} to solve the recursion, we get
\begin{align} \label{eqn:proof2.3}
\begin{split}
\left\|X^f_{\ell}-X^{'f}_{\ell}\right\| \leq L  \left(A_2 + \frac{\pi n_c}{\delta_c}\right)\epsilon \left(\prod_{\ell'=1}^{\ell-1} F_{\ell'}\right) \sum_{g=1}^{F_{0}}\left\|X^g \right\|.
\end{split}
\end{align}

To arrive at the result of Theorem 1, we  evaluate \eqref{eqn:proof2.3} with $\ell=L$ and substitute it into \eqref{eqn:proof2.0} to obtain
\begin{align} \label{eqn:proof2.4}
\begin{split}
\left\|Y-Y'\right\|^2 &= \sum_{f=1}^{F_L} \left\|X^f_{L} - X^{'f}_{L}\right\|^2\\
&\leq \sum_{f=1}^{F_L} \left(L \left(A_2 + \frac{\pi n_c}{\delta_c}\right)\epsilon \left(\prod_{\ell=1}^{L-1} F_{\ell}\right) \sum_{g=1}^{F_{0}}\left\|X^g \right\| \right)^2.
\end{split}
\end{align}
Finally, since $F_0 = F_L = 1$ and $F_\ell = F$ for $1 \leq \ell \leq L-1$, 
\begin{align} \label{eqn:proof2.5}
\begin{split}
\left\|Y-Y'\right\| \leq L F^{L-1}\left(A_2 + \frac{\pi n_c}{\delta_c}\right)\epsilon \left\|X \right\|\ .
\end{split}
\end{align}

\section{Proof of Lemma 1}

In order to prove Lemma \ref{lemma_random_graph}, we need the following lemma.

\begin{lemma} \label{lemma_filter_random_graph}
Let $\bar{\bbG}_n$ be an undirected graph with $\bar{\bbS}_n \in [0,1]^{n \times n}$. Let $\bbG_n$ be a stochastic graph obtained from $\bar{\bbG}_n$ as in \eqref{eqn:stochastic_graph}. Consider the graph convolution $\bbH(\bbS)$ with frequency response $h$ such that $h(\lambda/n)$ is constant for $|\lambda/n|<c$ and let $\bar{\bby}_n = \bbH(\bar{\bbS}_n)\bbx_n$ and ${\bby}_n = \bbH({\bbS}_n)\bbx_n$ be the outputs of this filter on the graphs $\bar{\bbG}_n$ and $\bbG_n$ respectively. Then, under Assumptions \ref{as:filter_lipschitz},\ref{as:nonlinearity_lipschitz} and \ref{as:graphon_degree}, it holds with probability at least $1-\xi$
\begin{align*}
\begin{split}
\|\bby_n-\bar{\bby}_n\| \leq \left(A_2 + \frac{\pi {n}^{(q)}_c}{{\delta}^{(pq)}_c}\right) \left({2\sqrt{\log{2n/\xi}}}/{\sqrt{n}}\right)\|\bbx_n\| 
\end{split}
\end{align*}
where ${n}^{(q)}_c$ and ${\delta}^{(pq)}_c$ are as in Definition \ref{def:constants} for $\bbW^{(p)}=\bar{\bbW}_n$, the graphon induced by $\bar{\bbG}_n$, and $\bbW^{(q)}={\bbW}_n$, the graphon induced by ${\bbG}_n$ \eqref{eqn:graphon_ind}. 
\end{lemma}
\begin{proof}
Following the same reasoning used in the proof of Theorem \ref{thm:wft_stability}, we analyze low-pass filters $h(\lambda/n)=0$ for $|\lambda/n|<c$. The result of the lemma then follows from the fact that the filters we consider are the sum of low-pass filters with constant filters, and constant filters are always transferable. Also similarly to what we did in the proof of Theorem \ref{thm:wft_stability}, we use the triangle inequality to bound $\|\bar{\bby}_n-\bby_n\|$ as
\begin{align*}
\begin{split}
\|\bbH(\bar{\bbS}_n)\bbx_n-\bbH(\bbS_n)\bbx_n\|  &\leq \left\|\sum_i \left(h(\bar{\lambda}_i/n)-h(\lambda_i/n)\right)[\hat{\bar{\bbx}}_n]_i \bar{\bbv}_i\right\| \mbox{\bf (2.1)}\\
&+\left\|\sum_i h(\lambda_i)\left( [\hat{\bar{\bbx}}_n]_i\bar{\bbv}_i - [\hat{{\bbx}}_n]_i{\bbv}_i \right)\right\| \mbox{\bf (2.2)}
\end{split}
\end{align*}
where $\bar{\lambda}_i$ and $\bar{\bbv}_i$ are the eigenvalues and eigenvectors of $\bar{\bbS}_n$, $\lambda_i$ and $\bbv_i$ those of $\bbS_n$, $[\hat{\bar{\bbx}}_n]_i = \bar{\bbv}_i^\Tr \bbx_n$ and $[\hat{{\bbx}}_n]_i = {\bbv}_i^\Tr \bbx_n$.

Using the Lipschitz property of the filter in the passing band together with Prop. \ref{prop:eigenvalue_diff} and the Cauchy-Schwarz inequality, we can rewrite {\bf (2.1)} as
\begin{align*}
\begin{split}
\|\bbH(\bar{\bbS}_n)\bbx_n-\bbH(\bbS_n)\bbx_n\|  &\leq \left\|\sum_i \left(h(\bar{\lambda}_i/n)-h(\lambda_i/n)\right)[\hat{\bar{\bbx}}_n]_i \bar{\bbv}_i\right\| \\
&\leq A_2 \dfrac{\|\bar{\bbS}_n-\bbS_n\|}{n} \|\bbx_n\|\ .
\end{split}
\end{align*}
The term $\|\bar{\bbS}_n-\bbS_n\|$ can be bounded by \cite[Theorem 1]{chung2011spectra} with probability $1-\xi$ provided that $d_{\bbG_n}$, the maximum expected degree of the graph $\bbG_n$, satisfies $d_{\bbG_n} > 4 \log (2n/\xi)/9$. This is the case for graphs $\bbG_n$ for which $n$ satisfies Assumption \ref{as:graphon_degree}.
To see this, write
\begin{align*}
&\frac{d_{\bbG_n}}{n} = \frac{1}{n} \max_i \left(\sum_{j=1}^n[\bar{\bbS}_n]_{ij}\right) = \frac{1}{n}\max_i \left(\sum_{j=1}^n \bbW(u_i,u_j)\right) \\
&= \max_{u \in [0,1]} \int_0^1 \bbW(u,v) dv \geq \max_{u \in [0,1]} \left(\int_0^1 \bbW(u,v) dv - \int_0^1 |\bbD(u,v)| dv\right)
\end{align*}
where $\bbD(u,v)$ is the degree function of $\bbW$, i.e., $\bbD(u,v)= \int_0^1 \bbW(u,v) dv$. Using the inverse triangle inequality for the maximum, we get
\begin{align*}
\frac{d_{\bbG_n}}{n} &\geq \max_{u \in [0,1]} \int_0^1 \bbW(u,v) dv - \max_{u \in [0,1]}\int_0^1 |\bbD(u,v)| dv \\
&= d_\bbW - \max_{u \in [0,1]}\int_0^1 |\bbD(u,v)| dv\ .
\end{align*}
Hence, it suffices to find an upper bound for the maximum of the integral of $|\bbD(u,v)|$. Let $F(u) = 1$. Then,
\begin{align*}
\max_u \int_0^1 |\bbD(u,v)|dv &= \max_u \int_0^1 |\bbD(u,v)F(v)|dv \\
&\leq  \sqrt{\int_0^1 \left(\int_0^1 |\bbD(u,v)F(v)|dv\right)^2 du} \\
&\leq \sqrt{\int_0^1 \int_0^1 |\bbD(u,v)|^2dv\int_0^1|F(v)|^2dvdu} \\
&= \sqrt{\int_0^1 \int_0^1 |\bbD(u,v)|^2dvdu}
\end{align*}
where the first inequality follows from the fact that the $L^2$ norm dominates the $L^\infty$ norm, and the second from Cauchy-Schwarz. Given the graphon's Lipschitz property, we know that $|\bbD(u,v)|$ is at most $2A_1/n$.  Therefore,
\begin{align*}
\max_u \int_0^1 |\bbD(u,v)|dv \leq \sqrt{\int_0^1 \int_0^1 \frac{4A_1^2}{n^2}dvdu} = \frac{2A_1}{n}
\end{align*}
and, thus,
\begin{align*}
{d_{\bbG_n}} \geq n d_\bbW - {2A_1} 
\end{align*}
which by Assumption \ref{as:graphon_degree} entails $d_{\bbG_n} > 4 \log (2n/\xi)/9$.

Given that Assumption \ref{as:graphon_degree} is satisfied, we can write
\begin{align} \label{eqn_lemma1_1}
\begin{split}
\|\bbH(\bar{\bbS}_n)\bbx_n-\bbH(\bbS_n)\bbx_n\|
&\leq A_2 \dfrac{2\sqrt{n\log(2n/\xi)}}{n} \|\bbx_n\|\ \\
&\leq A_2 \dfrac{2\sqrt{\log(2n/\xi)}}{\sqrt{n}} \|\bbx_n\|.
\end{split}
\end{align}

For \textbf{(2.2)}, we use the triangle inequality to decompose the norm difference as:
\begin{align*}
\begin{split}
&\left\|\sum_i h(\lambda_i/n)\left( [\hat{\bar{\bbx}}_n]_i\bar{\bbv}_i - [\hat{{\bbx}}_n]_i{\bbv}_i \right)\right\|
\\
&= \left\|\sum_{i} h(\lambda_i/n) \left([\hat{\bar{\bbx}}_n]_i\bar{\bbv}_i + [\hat{\bar{\bbx}}_n]_i{\bbv}_i - [\hat{\bar{\bbx}}_n]_i{\bbv}_i -[\hat{{\bbx}}_n]_i{\bbv}_i \right) \right\| \\
&\leq \left\|\sum_{i} h(\lambda_i/n) [\hat{\bar{\bbx}}_n]_i(\bar{\bbv}_i-{\bbv}_i) \right\| + \left\|\sum_{i} h(\lambda_i/n) {\bbv}_i \langle \bbx_n, \bar{\bbv}_i - {\bbv}_i \rangle \right\| \\
&\leq 2\sum_{i} \|h(\lambda_i/n)\| \|X\| \|\bar{\bbv}_i - {\bbv}_i\|
\end{split}
\end{align*}
where the last inequality follows from the Cauchy-Schwarz inequality.

Prop. \ref{thm:davis_kahan} gives an upper bound to the norm difference between the eigenvectors, by which we obtain:
\begin{align*}
\begin{split}
\left\|\sum_i h\left({\lambda_i}/{n}\right)\left( [\hat{\bar{\bbx}}_n]_i\bar{\bbv}_i - [\hat{{\bbx}}_n]_i{\bbv}_i \right)\right\| \leq \|\bbx_n\| \sum_{i} \|h&\left({\lambda_i}/{n}\right)\| \\
&\times \frac{\pi\|\bar{\bbS}_n-\bbS_n\|}{d_i}
\end{split}
\end{align*}
where $d_i$ is the minimum between $\min(|\bar{\lambda}_i - \lambda_{i+1}|,|\bar{\lambda}_i-\{\lambda_{i-1}|)$ and $\min(|\lambda_i - \bar{\lambda}_{i+1}|,|\lambda_i-\bar{\lambda}_{i-1}|)$ for each $i$. Due to the fact that the eigenvalues of the graphon induced by a graph are equal to the eigenvalues of the graph divided by $n$ \cite[Lemma 1]{ruiz2020graphon}, and by Definition \ref{def:constants}, we have $n \delta^{(pq)}_c \leq d_i$ for all $i$. Leveraging this together with \cite[Theorem 1]{chung2011spectra}, it thus holds with probability $1-\xi$ that 
\begin{align} \label{eqn_lemma1_2}
\begin{split}
\left\|\sum_i h\left({\lambda_i}/{n}\right)\left( [\hat{\bar{\bbx}}_n]_i\bar{\bbv}_i - [\hat{{\bbx}}_n]_i{\bbv}_i \right)\right\| &\leq \|\bbx_n\| n_c^{(q)} \frac{2\pi\sqrt{n\log (2n/\xi)}}{n\delta_c^{(pq)}} \\
&\leq \|\bbx_n\| n_c^{(q)} \frac{2\pi\sqrt{\log (2n/\xi)}}{\delta_c^{(pq)}\sqrt{n}}
\end{split}
\end{align}
where $n_c^{(q)}$ is the number of eigenvalues of $\bbS_n$ in the passing band of $h$. Putting \eqref{eqn_lemma1_1} and \eqref{eqn_lemma1_2} together completes the proof.
\end{proof}

\bibliographystyle{IEEEbib}
\bibliography{myIEEEabrv,bib-graphon,graphMachineLearningBiblio,edgeNetsBiblio}

\end{document}